\documentclass{article}

\usepackage[preprint]{neurips_2026}

\usepackage[utf8]{inputenc} 
\usepackage[T1]{fontenc}    
\usepackage{hyperref}       
\usepackage{url}            
\usepackage{booktabs}       
\usepackage{amsfonts}       
\usepackage{nicefrac}       
\usepackage{microtype}      
\usepackage{xcolor}         

\usepackage{graphicx}
\usepackage{amsmath,amsthm}
\usepackage{amssymb}
\usepackage{xcolor}
\usepackage{makecell}
\usepackage{enumitem}
\usepackage{tikz}
\usepackage{bm} 
\usepackage{amsmath} 
\usepackage{amssymb} 
\usepackage{amsthm} 
\usepackage[boxed]{algorithm}
\usepackage[noend]{algpseudocode}
\usepackage{braket}

\algnewcommand\algorithmicinput{\textbf{Input:}}
\algnewcommand\Input{\item[\algorithmicinput]}
\algnewcommand\algorithmicoutput{\textbf{Output:}}
\algnewcommand\Output{\item[\algorithmicoutput]}
\algnewcommand\algorithmicalgorithm{\textbf{Algorithm:}}
\algnewcommand\Algorithm{\item[\algorithmicalgorithm]}

\newtheorem{theorem}{Theorem}[section]
\newtheorem{lemma}{Lemma}[section]

\newtheorem{definition}{Definition}[section]
\newtheorem{problem}{Problem}[section]

\DeclareMathOperator{\diag}{diag}

\DeclareMathOperator{\poly}{poly}

\DeclareMathOperator{\softmax}{softmax}
\DeclareMathOperator{\Att}{Att}

\newcommand\veck{\boldsymbol{\mathrm{k}}}

\newcommand\vecq{\boldsymbol{\mathrm{q}}}

\newcommand\vecu{\boldsymbol{\mathrm{u}}}
\newcommand\vecv{\boldsymbol{\mathrm{v}}}

\newcommand\matA{\boldsymbol{\mathrm{A}}}
\newcommand\matB{\boldsymbol{\mathrm{B}}}
\newcommand\matC{\boldsymbol{\mathrm{C}}}
\newcommand\matD{\boldsymbol{\mathrm{D}}}

\newcommand\matH{\boldsymbol{\mathrm{H}}}

\newcommand\matK{\boldsymbol{\mathrm{K}}}

\newcommand\matM{\boldsymbol{\mathrm{M}}}

\newcommand\matQ{\boldsymbol{\mathrm{Q}}}

\newcommand\matU{\boldsymbol{\mathrm{U}}}
\newcommand\matV{\boldsymbol{\mathrm{V}}}

\newcommand\matAtilde{\widetilde{\boldsymbol{\mathrm{A}}}}

\newtheorem{keylemma}{Key Lemma}[section]
\newtheorem{mylemma}[keylemma]{Lemma}

\title{\hspace{-3pt}Approaching I/O-optimality for Approximate Attention\hspace{-3pt}}
\date{}
\author{
  P\'al Andr\'as Papp\thanks{Contacts: pal.andras.papp@huawei.com,  aleksandros.sobczyk@huawei.com,  anastasios.zouzias@huawei.com} 
  \\
  \And
  Aleksandros Sobczyk
  \\
  \\
  Huawei Technologies Switzerland AG \\
  Zurich, Switzerland \vspace{-2pt}
  \And
  Anastasios Zouzias
}

\begin{document}

\maketitle

\begin{abstract}
    We revisit the I/O complexity of \emph{attention} in large language models. Given query--key--value matrices $\matQ,\matK,\matV\in\mathbb{R}^{n\times d}$, and a machine with fast memory size $M$, the goal is to compute the ``attention matrix'' $\matA=\softmax(\matQ\matK^{\top}/\sqrt{d})\matV$ with the minimal number of data transfers between fast and slow memory. Existing methods in the literature, most notably FlashAttention and its variants, incur an I/O cost that depends quadratically on the context length $n$, while a trivial lower bound only requires $\Omega(nd)$ I/O's to read the inputs and write the output. In this work we show that the dependence on $n$ can be reduced to almost-linear, by developing I/O-efficient algorithms inspired by the recent approximate attention method of Alman and Song~\cite[NeurIPS'23]{alman2023fast}. We also provide (algorithm-specific) lower bounds in different parameter regimes, which nearly match the corresponding upper bounds, indicating that the proposed algorithms are indeed close to I/O-optimal.
\end{abstract}

\section{Introduction}
Since its discovery \cite{vaswani2017attention}, the attention mechanism has reshaped the entire field of machine learning, and, in particular, Large Language Models (LLMs). Given a sequence of tokens of length $n$, and a feature dimension $d$,
the attention mechanism takes as input three matrices $\matQ\in\mathbb{R}^{n\times d}$ (query), $\matK\in\mathbb{R}^{n\times d}$ (key), and $\matV\in\mathbb{R}^{n\times d}$ (value), and returns the \emph{attention matrix}:
\begin{align}
\matA = \Att(\matQ,\matK,\matV):=\matD^{-1}\exp(\matM)\matV,
\label{eq:attention_matrix}
\end{align}
where $\matM=\tfrac{1}{\sqrt{d}}\matQ\matK^{\top}$, $\exp(\matM)$ is element-wise exponentiation, and $\matD$ is a diagonal matrix with $\matD_{i,i}=\sum_{j=1}^n \exp(\matM_{i,j})$ (in the literature $\matD^{-1}\exp(\matM)$ is often written as $\softmax(\tfrac{1}{\sqrt{d}}\matQ\matK^{\top})$).

The straightforward algorithm, which first computes the matrix $\matQ\matK^{\top}$ using standard matrix multiplication, has an arithmetic complexity of 
$O(n^2d)$. 
In the so-called ``long-context'' regime, where $n$ can grow very large, the quadratic complexity often forms a computational bottleneck.
Overcoming the $O(n^2)$ cost of attention is one of the most interesting computational challenges of the last decade. Indeed, the literature on fast and efficient attention algorithms is quite extensive, where common techniques to reduce the arithmetic intensity include approximation algorithms \cite{choromanski2021rethinking,daras2020smyrf,han2024hyperattention,kitaev2020reformer,xiong2021nystromformer}, sparsification \cite{child2019generating,zaheer2020big}, or replacing the softmax with a kernel function \cite{katharopoulos2020transformers}. A particularly important work on approximate attention is the recent algorithm by Alman and Song~\cite{alman2023fast}, which also comes with strong theoretical bounds on the approximation guarantee.

The aforementioned works can considerably reduce the arithmetic complexity of attention. However, in modern systems, the computational runtimes also depend on several other aspects. One particularly important factor is the amount of data movements within the memory hierarchy. 
Due to this, there is a line of research analyzing computations with respect to their \emph{I/O complexity} or I/O cost, i.e.\ the data movement they require when executed in a two-level memory hierarchy with a fast memory of limited capacity $M$ (e.g.\ cache) and a slow memory of unlimited capacity (e.g.\ RAM). For attention, if we simply use the I/O-optimal method for standard matrix multiplication, we obtain an I/O cost of $O(\frac{n^2 \cdot d}{\sqrt{M}})$. However, the celebrated FlashAttention algorithm by Dao et al~\cite{dao2022flashattention}, and follow-up works \cite{dao2024flashattention2,shah2024flashattention3}, allow to execute the same algorithm with an I/O cost of only $O(\frac{n^2 \cdot d^2}{M})$ by fusing the two multiplications and re-arranging the order of execution. This is indeed an improvement over the naive method when $M > d^2$, i.e.\ the cache size is sufficiently large, which is often the case in practice. The work of Saha and Ye~\cite{saha2024complexity} also proves tight (conditional) lower bounds for both of these cases up to constant factors, showing that these algorithms (naive matrix multiplication and FlashAttention) are indeed I/O-optimal in the corresponding cache size regimes.

In this work, we combine the two lines of research above: we study the most I/O-efficient way to execute attention if we use the state-of-the-art algorithm from~\cite{alman2023fast}. Our results show that this approximate attention algorithm can significantly outperform classical attention not only in terms of computational, but also in terms of I/O complexity. In particular, its asymptotic I/O cost is also significantly lower than that of the FlashAttention algorithm for classical attention.

\subsection{Contributions}
Our main results are tight and nearly-tight bounds for the I/O complexity of attention when computed using the algorithm of Alman and Song~\cite{alman2023fast}. It turns out to that this I/O complexity can behave rather differently depending on the relations between several parameters: the feature dimension $d$ of the input matrices, the polynomial degree $g$ used to approximate the exponential function in~\cite{alman2023fast}, the size $r={d+g \choose g}$ of the approximation matrices in~\cite{alman2023fast}, and the fast memory capacity $M$. For clarity, the role of these parameters is summarized in Table~\ref{tab:params}. We distinguish the following cases:
\begin{itemize}[leftmargin=1.6em, topsep=1pt, itemsep=4pt]
 \item \textbf{Case I:} when $M = \Omega (d \cdot r)$, i.e.\ the fast memory is so large that it can essentially fit all of intermediate matrix $\matU_2^{\top}\matV$ from~\cite{alman2023fast}, up to constant factors;
 \item \textbf{Case II:} when $M = o (d \cdot r)$, and also $g = o(\log{M})$, i.e.\ the fast memory is more than exponentially larger than the degree $g$ of the approximating polynomial;
 \item \textbf{Case III:} when $M = o (d \cdot r)$, and $g = \Omega(\log{M})$ but $g = o(\sqrt{M})$, i.e.\ the fast memory is between exponential and quadratic in the degree $g$;
 \item \textbf{Case IV:} when $M = o (d \cdot r)$ and $g = \Omega(\sqrt{M})$, i.e.\ the fast memory is at most quadratic in $g$.
\end{itemize}

We prove different I/O bounds for these cases. In Case I, when the fast memory essentially fits the intermediate matrix $\matU_2^{\top}\matV$ in~\cite{alman2023fast}, a simpler strategy produces tight bounds up to a constant factor.

\begin{theorem}[Case I.] \label{th:case1}
If $M = \Omega (d \cdot r)$, then the optimal I/O complexity of Approximate Attention is both upper and lower bounded by $\Theta(n \cdot d)$.
\end{theorem}

In Case II, when the degree of the approximating polynomial is small (compared to the fast memory size), the upper bounds we obtain are somewhat more complex.

\begin{theorem}[Case II.] \label{th:case2}
If $M = o (d \cdot r)$ and $g = o(\log{M})$, then the optimal I/O complexity of Approximate Attention is upper bounded by
\[ O\left(\frac{n \cdot r \cdot d \cdot (4e^2)^g}{M^{\frac{g}{g+1}}}\right) \, ,\]
and lower bounded by
\[ \Omega\left(\frac{n \cdot r \cdot d \cdot g}{M^{\frac{g}{g+1}}}\right) \, . \]
\end{theorem}

When the polynomial degree is larger or fast memory capacity is small, the problem becomes even more technical. In these cases we need an additional assumption $d \geq 5g$ for our lower bounds.

\begin{theorem}[Case III.] \label{th:case3}
If $M = o (d \cdot r)$ and $\Omega(\log{M}) \leq g \leq o(\sqrt{M})$, then the optimal I/O complexity of Approximate Attention is upper bounded by
\[ O\left(\min \left( \frac{n \cdot r \cdot d^2}{M}, \, \frac{n \cdot r \cdot d}{\sqrt{M}}\right) \right) \, , \]
and when $d \geq 5g$, it is also lower bounded by $\Omega\left( \frac{n \cdot r \cdot d \cdot g}{M} \right)$. 
\end{theorem}

Finally, for tiny cache sizes $M=O(g^2)$, the problem behaves like standard matrix multiplication.

\begin{theorem}[Case IV.] \label{th:case4}
If $M = o (d \cdot r)$ and $g = \Omega(\sqrt{M})$, the optimal I/O complexity of Approximate Attention is upper and lower bounded by $\Theta\left(\frac{n \cdot r \cdot d}{\sqrt{M}} \right)$, where the lower bound also requires $d \geq 5g$. 
\end{theorem}

For convenience, we summarize the results in Table \ref{tab:attention-bounds}. 
Theorems \ref{th:case1}-\ref{th:case4} provide a comprehensive overview of the optimal I/O complexity of Approximate Attention. 
From a broader perspective, our results can also be understood as an I/O analysis of matrix multiplication in a special setting where one of the input matrices is sparse, or behaves unusually in some way regarding I/O. We believe that out proof techniques could also inspire further results in sparse attention computation or in entirely different domains.

\begin{table}[t]
\caption{Notation and summary of the main parameters used throughout our paper.
}
        \centering
        \vspace{-3pt}
        \small
        \renewcommand{\arraystretch}{1.35}
        \begin{tabular}{| c | c | m{7.7cm} |} 
        \hline
            Parameter & Origin & Description \\ 
         \hline\hline
          \makecell{$n$ \\ (sequence length)} & Input problem & Number of rows in $\matQ$, $\matK$, $\matV$. \\ \hline
          \makecell{$d$ \\ (feature dimension)} & Input problem & Number of columns in $\matQ$, $\matK$, $\matV$. Often assumed to be $O(\log n)$ for long-context attention~\cite{alman2023fast}. \\ \hline
          \makecell{$M$ \\ (fast memory\\capacity)} & Architecture & The amount of available fast memory (i.e.\ cache size), which determines the optimal I/O cost. \\ \hline
          \makecell{ $g$ \\ (degree of approxi-\\mating polynomial) } & \makecell{Algorithm of \\ Alman \& Song} & Degree of the polynomial used to approximate the exponentiation. Depends on the accepted approximation error and the magnitude of the entries. Assumed to be $o(\log n)$ in~\cite{alman2023fast}.
          \\ \hline
          \makecell{$r$ \\(dimension of the \\approximating \\ matrices)} & \makecell{Algorithm of \\ Alman \& Song} & The maximal number of terms in a polynomial of degree $g$ on $d$ variables. Its value equals ${d+g \choose g}$. It determines the number of columns in the approximation matrices $\matU_1$ and $\matU_2$. \\ \hline
          \makecell{$w$ \\(generating set size)} & Our I/O analysis & The number of entries we decide to load from slow memory in a row of $\matQ$ (or $\matK$), used in our proofs. If $w \leq d$, this allows to compute at most ${w+g \choose g}$ entries in a row of $\matU_1$ (or $\matU_2$). \\ \hline
        \end{tabular}
        \vspace{-4pt}
  \label{tab:params}
\end{table}

\begin{table}[ht]
\centering
\caption{Summary of the reported I/O complexity bounds.}
\renewcommand{\arraystretch}{1.4}
\small
\begin{tabular}{l l l l}
\hline
& \ \ \textbf{Regime}  & \textbf{Upper Bound} & \textbf{Lower Bound} \\
\hline\hline
Thm.\@\ref{th:case1} (Case I) &
\begin{tabular}{l}$M = \Omega(d  r)$\end{tabular}&
$O(n  d)$ &
$\Omega(n  d)$  \\\hline
Thm.\@\ref{th:case2} (Case II)&
\begin{tabular}{l}
$M = o(d  r)$\\
$g = o(\log M)$
\end{tabular}&
$O\left(\dfrac{n  r  d  (4e^2)^g}{M^{\frac{g}{g+1}}}\right)$ &
$\Omega\left(\dfrac{nrdg}{M^{\frac{g}{g+1}}}\right)$  $^\dagger$\\
\hline
Thm.\@\ref{th:case3} (Case III)&
\begin{tabular}{l}$M=o(dr)$\\$\Omega(\log M)\le g\le o(\sqrt{M})$\end{tabular}&
$O\left( \min\left( \dfrac{n  r  d^2}{M}, \dfrac{n  r  d}{\sqrt{M}} \right) \right)$ &
$\Omega\left( \dfrac{n  r  d  g}{M} \right)$  $^\dagger$\\
\hline
Thm.\@\ref{th:case4} (Case IV) &
\begin{tabular}{l}
$M = o(d  r)$
\\ $g = \Omega(\sqrt{M})$
\end{tabular}
&
$O\left( \dfrac{n  r  d}{\sqrt{M}} \right)$ &
$\Omega\left( \dfrac{n  r  d}{\sqrt{M}} \right)$  $^\dagger$\\
\hline
\multicolumn{4}{p{0.9\textwidth}}{\footnotesize $^\dagger$ The lower bounds in Cases II, III, and IV are obtained using $S$-partitions and they are algorithm-specific. All other bounds are algorithm-agnostic.}
\end{tabular}
\label{tab:attention-bounds}
\end{table}

\subsection{Parameters and discussion}

We discuss the bounds in more detail with some example parameters for clarity.

For the polynomial degree $g$, Alman and Song~\cite{alman2023fast} consider $g=o(\log n)$, which holds, for example, when $B=o(\log n)$ and $\epsilon=1/n$; together with $d=O(\log n)$, this is already enough to establish $r=n^{o(1)}$ in their analysis. 
In practice, however, $B$ and $\epsilon$ may not grow as a function of the context length $n$ at all. Instead, 
for long--context attention, it is reasonable to treat these parameters independently, or even as constants.

If $g=O(1)$, then the upper and lower bounds in Theorems~\ref{th:case2} and \ref{th:case3} are also tight up to a constant factor. This is easy to see for Theorem~\ref{th:case2}. In Theorem~\ref{th:case3}, the gap between the bounds is at most $\frac{\sqrt{M}}{g}$; with $g=\Omega(\log M)$, we can upper bound this by $2^{O(g)}$, which is again a constant. Theorems~\ref{th:case1} and \ref{th:case4} are less interesting from this angle: these bounds are always tight.


In order to compare the bounds to FlashAttention, let us consider $d^2 < M < n \cdot d$; this is the range where FlashAttention improves upon naive matrix multiplication~\cite{saha2024complexity}.

For Case I, assume the aforementioned regime from~\cite{alman2023fast}, namely, $d=O(\log(n))$, $B=o(\sqrt{\log(n)})$ and $\epsilon=1/n$, which ensures that $r=n^{o(1)}$. The condition of Case I is $M = \Omega (d \cdot r)$; hence we are in this setting if $M$ is larger than $d \cdot r= n^{o(1)}$. Theorem \ref{th:case1} reports an $O(n \cdot d)$ upper bound, which matches the trivial lower bound of $\Omega(n \cdot d)$, while FlashAttention incurs a cost of $\Theta(\frac{n^2\cdot d^2}{M})$~\cite{saha2024complexity}. For all $M\in \Omega(d \cdot r) \cap o(n \cdot d)$, the latter becomes $\omega(n \cdot d)$, which implies a strict asymptotic separation. For example, for $M=n^{0.1} \! \cdot d$, our upper bound remains $O(n \cdot d)$, but the I/O complexity of FlashAttention is $\Theta(n^{1.9} \! \cdot d)$. The factor of difference $\frac{n \cdot d}{M}$ between our method and FlashAttention becomes more significant as $n$ grows.

For the remaining cases, the bounds we obtain are again almost-linear in $n$:
for Cases II, III and IV, respectively, we get upper bounds of
\[ O\left( \frac{n ^{1+o(1)} \cdot d}{M^{\frac{g}{g+1}}} \right), \quad O\left( \min \left( \frac{n ^{1+o(1)} \cdot d^2}{M}, \, \frac{n ^{1+o(1)} \cdot d}{\sqrt{M}} \right) \right) \quad \text{and} \quad O\left( \frac{n ^{1+o(1)} \cdot d}{\sqrt{M}} \right) \, . \]
In particular, the first two of these expressions are at most $O \!\left( \frac{n ^{1+o(1)} \cdot d^2}{M} \right)$ in the given parameter regimes.
FlashAttention has a very similar I/O cost of $\Theta(\frac{n^2 \cdot d^2}{M})$, which depends quadratically on $n$; hence compared to it, our algorithms save at least an almost-linear factor of I/O cost in $n$.

Finally, we make a remark regarding the exponential gap (with respect to $g$) in Case II. Note that in that case $g=o(\log(M))$. This implies $(4e^2)^g=M^{o(1)}$, and, therefore, the corresponding gap between the upper and lower bound is in fact a factor of $M^{o(1)}$. Theorem \ref{th:case2} explicitly states the dependence on $g$ for completeness, but this observation can be used to simplify the bounds if required. 



\section{Preliminaries} \label{sec:prelim}

\subsection{Approximate attention}
\label{section:approximate_attention}


Many existing works on attention algorithms target the \emph{exact} computation of the attention matrix. While this is not necessarily ``unreasonable'', it raises some intricate questions. In particular, if attention can be computed exactly in $T$ steps, then $\exp(x)$ can be also computed exactly in $T+O(1)$ steps (see Appendix \ref{appendix:attention_implies_exponential}). However, assuming that the entire sequence of digits of $\exp(x)$ can be returned in a single ``time-step'' might raise concerns regarding the machine model; see e.g.\ the discussion in \cite{erickson2022smoothing}. To avoid such intricacies, we target \emph{approximate attention} algorithms, rather than exact ones.
We consider the following approximate attention definition from \cite{alman2023fast}.

\begin{problem}[Additive-error Approximate Attention \cite{alman2023fast}] \label{problem:approximate_attention}
Given $\epsilon>0$, and matrices $\matQ,\matK,\matV\in\mathbb{R}^{n\times d}$, return a matrix $\matAtilde$ which satisfies:
\begin{align*}
    \max_{i,j}\left|\matAtilde_{i,j}- \Att(\matQ,\matK,\matV)_{i,j}\right| \leq \epsilon.\end{align*}
\end{problem}
This specific definition has proven useful to obtain sharp bounds on the arithmetic complexity of attention. When $d=O(\log(n))$ and $\matQ,\matK,\matV$ have entries of magnitude $o(\sqrt{\log(n)})$, the authors of~\cite{alman2023fast} describe an algorithm with nearly-linear $n^{1+o(1)}$ complexity in the algebraic model; we describe this in detail in Section~\ref{sec:alman}. On the other hand, the authors also prove that if the magnitude of entries in $\matQ,\matK,\matV$ is $\Theta(\sqrt{\log(n}))$, then no algorithm can solve Problem \ref{problem:approximate_attention} up to $1/\poly(n)$  error in sub-quadratic time under the Strong Exponential Time Hypothesis (SETH).

\subsection{The algorithm of Alman and Song} \label{sec:alman}
In brief, the algorithm of Alman and Song~\cite{alman2023fast} that solves \ref{problem:approximate_attention} works as follows:
\begin{enumerate}[leftmargin=2.3em, topsep=1pt, itemsep=3pt]
    \item Compute the coefficients of a degree-$g$ polynomial $P(x)$, which approximates $\exp(x)$ in the domain $[-B,B]$, where $B$ is a bound on the magnitudes of the elements of $\matQ,\matK$.
    \item Construct two matrices $\matU_1,\matU_2\in\mathbb{R}^{n\times r}$ for some $r=n^{o(1)}$, such that $\matU_1\matU_2^\top = P(\matQ\matK^\top) \approx \exp(\matQ\matK^\top/\sqrt{d})$. Intuitively, the $i^{\text{th}}$ row of $\matU_1$ ($i^{\text{th}}$ row of $\matU_2$, respectively) is obtained from the $i^{\text{th}}$ row of $\matQ$ ($i^{\text{th}}$ row of $\matK$), via an expansion of the polynomial $P(\cdot)$. 
    \item Return $\matAtilde\gets \widetilde \matD^{-1}(\matU_1(\matU_2^\top \matV))$, where $\widetilde \matD=\diag(\matU_1(\matU_2^\top \boldsymbol{1}))$. 
\end{enumerate}

The key idea that makes this algorithm efficient is that the polynomial approximation completely removes the non-linearity of the $\softmax$. This allows to execute the matrix multiplication $\matU_2^\top \matV$ first, and then $\matU_1(\matU_2^\top \matV)$ afterwards, which altogether needs significantly fewer arithmetic operations.

Going in more detail, in step 2 we have that $\exp(\matQ\matK^{\top})_{i,j}=\exp(\vecq^\top \veck)$, where $\vecq^\top$ is the $i^{\text{th}}$ row of $\matQ$ and $\veck$ is the $j^{\text{th}}$ column of $\matK^{\top}$. When $\exp(x)$ is replaced by the polynomial $P(x)$, we can write:
\begin{align*}
    P(\vecq^\top \veck) &= P(q_1k_1+q_2k_2 +\ldots +q_dk_d) =
    \sum_{l=0}^g c_l(q_1k_1+q_2k_2 +\ldots +q_dk_d)^l.
\end{align*}
If we expand $(q_1k_1+q_2k_2 +\ldots +q_dk_d)^l$ for any $l \!\in \! \{1,\ldots, g\}$, we obtain a sum with $\binom{d+l-1}{l}$ additive terms. Each additive term is of the form $ (q_1k_1)^{l_1}\cdot (q_2k_2)^{l_2}\cdot \ldots \cdot (q_dk_d)^{l_d}$, where $l_1+l_2+\ldots+l_d=l$. 
Altogether, we get $r = \sum_{l=0}^{g} {d + l -1 \choose l}$ additive terms. Ultimately, since each term is the product of some powers of the elements of $\vecq$ and $\veck$, the polynomial $P(\vecq^\top \veck)$ can be written as a bilinear form $P(\vecq^\top \veck) = \vecu^\top \matC \vecv$, where $\vecu\in\mathbb{R}^{r}$ consists of products of (powers of) the elements of $\vecq$, while $\vecv\in\mathbb{R}^r$ consists of products of (powers of) the elements of $\veck$, and $\matC\in\mathbb{R}^{r\times r}$ is a diagonal matrix with scalar coefficients that depend on the polynomial $P(\cdot)$, but are independent from $n$.

To ensure $|P(x)-\exp(x)| \leq \epsilon \cdot  \exp(x)$, setting $g:=\Theta\left(\max\left\{\frac{\log(1/\epsilon)}{\log(\log(1/\epsilon)/B^2)}, B^2\right\}\right)$ is sufficient~\cite{aggarwal2022optimal,alman2023fast}. The coefficients of $P(\cdot)$ can be computed as $\poly(g)$-bit rationals in $\poly(g)$ time. 

Below we introduce a specific notation for the number of possible terms in a polynomial of degree $g$ on $w$ variables; this will play a central role in our analysis.

\begin{definition}
Let us define the function $\tau: \mathbb{Z}^+ \rightarrow \mathbb{Z}^+$ as a shorthand notation for
\[  \tau(w) := \sum_{l=0}^{g} \binom{w+l-1}{l} = {w+g \choose g}\, . \]
\end{definition}
Intuitively, ${w+l-1 \choose l}$ is a combination with repetition, i.e.\ the number of ways we can divide a total degree of $g$ among $w$ variables. The number of columns in $\matU_1$ and $\matU_2$ is $r=\tau(d)$. The closed form expression of the sum in the definition can be shown through an induction; we defer the proof to Appendix~\ref{app:tech_lemma}. We note that the work of~\cite{alman2023fast} uses a looser bound of ${2(w+g) \choose 2g}$ to bound $r$.

Note that the algorithm above executes matrix multiplications with a standard inner product-based algorithm. Previous works on the I/O analysis of attention also focus on standard matrix multiplication~\cite{saha2024complexity}. It is well-known that ``Strassen-like'', fast matrix multiplication algorithms (see e.g. \cite{alman2025more,alman2025improving} and references therein) can achieve better computational and I/O complexity~\cite{demaine2018fine}. 
However, these algorithms are often known to be rather impractical, and their analysis is more complex than the standard algorithm, especially when rectangular matrices are involved. 

Throughout the paper, we analyze I/O-efficient algorithms based on the aforementioned Approximate Attention method. Note that our upper bounds naturally carry over to solving Problem~\ref{problem:approximate_attention} in general. On the other hand, lower bounds in standard I/O-complexity models are inherently coupled to a concrete algorithm, so our bounds here are specific to this Approximate Attention technique.

\subsection{Model of computation}

The most prominent model to analyze the I/O complexity of computations is the red-blue pebble game of Hong and Kung~\cite{hong1981complexity}. In this setting, a computation is captured as a directed acyclic graph (DAG), where the nodes represent operations, and the directed edges represent data dependencies, i.e.\ that the output of an operation is required as an input for another operation. The computational DAG corresponding to the approximate attention algorithm from~\cite{alman2023fast} is illustrated in Figure~\ref{fig:graph}. Note that the algorithm only uses basic algebraic operators $\{+,-,\times,/\}$, and square roots, hence avoiding the intricacies of exact exponentiation that we mentioned in Section \ref{section:approximate_attention}.

The red-blue pebble game models the execution of the computation in a two-level memory hierarchy: we have a fast memory with limited capacity $M$, and slow memory of unlimited capacity. If the output of an operation is currently stored in fast (slow) memory, then this is indicated by having a red (blue) pebble on the node. The number of red pebbles can not exceed the cache capacity $M$ at any time. Loading a value from (respectively, saving a value to) slow memory then corresponds to placing a red (blue) pebble on a node that already has a blue (red) pebble. A new output value can only be computed (placing a red pebble on the node) if all the parents of the node have a red pebble.

The execution of the computation is modeled by a sequence of steps, where each step can load, save, compute or delete a value. Initially, the source nodes of the DAG have a blue pebble, and the process finishes when all the sink nodes have a blue pebble. Since the goal of the model is to capture I/O complexity, the cost of the pebbling sequence is the total number of save and load operations executed; compute and delete steps are considered free.

Red-blue pebbling has been thoroughly analyzed in several works, and widely used to establish upper and lower bounds on the I/O cost of specific computations~\cite{boehnlein2025red, demaine2018red,hong1981complexity,kwasniewski2019red,papp2025impact,papp2020on,sobczyk2026complexity}. 

\section{The I/O Cost of Approximate Attention} \label{sec:proof}
In this section, we outline the main intuition behind the bounds in Theorems \ref{th:case1}-\ref{th:case4}, and the main ingredients of the proofs. The detailed proofs are deferred to Appendices~\ref{app:keylemma}, \ref{app:case3} and \ref{app:case4}.

In our I/O analysis, we focus on the dominant step of computing $\matU_1(\matU_2^\top \matV)$; the other parts of the computation are negligible compared to this. $\matU_1(\matU_2^\top \boldsymbol{1})$ can be computed by executing a further column to $\matV$, which does not increase the cost asymptotically. The normalization step is executed on the output matrix of size $n \times d$; even if each entry is naively re-loaded and saved, this only adds $O(n \cdot d)$, which is a lower bound on the cost in any case.

The computation $\matU_1(\matU_2^\top \matV)$ consists of two consecutive matrix multiplications; we denote the intermediate matrix by $\matH := \matU_2^\top \matV$. Throughout the proofs, we focus on the second multiplication $\matU_1 \matH$; note that $\matU_1 \in\mathbb{R}^{n\times r}$ and $\matH \in\mathbb{R}^{r\times d}$. The lower and upper bounds are then easy to extend to the entire computation $\matU_1(\matU_2^\top \matV)$. Considering the sub-problem $\matU_1 \matH$ separately is in fact a restriction of I/O strategies: the entries in $\matH$ here are all inputs that need to be loaded from slow memory. This rules out ``fusion-based'' strategies that never save the entries of $\matH$, but instead immediately go on to multiply them with $\matU_1$. Intuitively, these strategies can be ignored because $\matH$ is smaller than $\matQ$, $\matK$ and $\matV$, so saving/loading it does not affect the magnitude of the total cost.

\subsection{Intuition: optimal tile design}

In standard matrix multiplication, the key to optimal I/O strategies  is to divide the matrix into rectangular tiles to maximize reusing the same inputs and outputs. In general, when multiplying matrices $\matA\in\mathbb{R}^{m_1\times m_2}$ and $\matB\in\mathbb{R}^{m_2\times m_3}$, the optimal strategy uses square tiles. That is, we read the inputs in a square-shaped $\Theta(\sqrt{M}) \times \Theta(\sqrt{M})$ tile in $\matA$, and another square-shaped $\Theta(\sqrt{M}) \times \Theta(\sqrt{M})$ tile in $\matB$, and we execute the corresponding part of the matrix multiplication. This gives us a partial sum for each entry in a $\Theta(\sqrt{M}) \times \Theta(\sqrt{M})$ tile of the output matrix, which we can add (and save) to the already aggregated values for this specific output entry. Each such tiling step only requires a fast memory capacity of $\Theta(M)$, it has an I/O cost of only $\Theta(M)$, and covers $\Theta(\sqrt{M})^3 = \Theta(M^{3/2})$ of the scalar multiplications that compose the entire operation. With a total of $m_1 \cdot m_2 \cdot m_3$ multiplications in the entire operation, this allows for an algorithm with an I/O cost of $\Theta\left(\frac{m_1 \cdot m_2 \cdot m_3}{\sqrt{M}}\right)$. This general tiling idea is illustrated in Figure~\ref{fig:matmul} for the matrix multiplication $\matU_1 \matH$.

\begin{figure}[t!]
    \centering
    \begin{minipage}{.44\textwidth}
      \centering
      \vspace{22pt}
      \resizebox{0.85\textwidth}{!}{
      \begin{tikzpicture}

    \draw (0pt,0pt) -- (0pt,20pt) -- (50pt,20pt) -- (50pt,0pt) -- cycle;

    \draw (-10pt,40pt) -- (-10pt,60pt) -- (60pt,60pt) -- (60pt,40pt) -- cycle;

    \draw (-45pt,80pt) -- (-45pt,100pt) -- (15pt,100pt) -- (15pt,80pt) -- cycle;

    \draw (-40pt,120pt) -- (-40pt,140pt) -- (10pt,140pt) -- (10pt,120pt) -- cycle;

    \draw (55pt,80pt) -- (55pt,100pt) -- (95pt,100pt) -- (95pt,80pt) -- cycle;

    \draw (40pt,120pt) -- (40pt,140pt) -- (110pt,140pt) -- (110pt,120pt) -- cycle;

    \draw (0pt,160pt) -- (0pt,180pt) -- (60pt,180pt) -- (60pt,160pt) -- cycle;

    \draw (5pt,200pt) -- (5pt,220pt) -- (55pt,220pt) -- (55pt,200pt) -- cycle;

    \draw (90pt,160pt) -- (90pt,180pt) -- (140pt,180pt) -- (140pt,160pt) -- cycle;

	\node[anchor=center] at (25pt,10pt) {output};
    \node[anchor=center] at (-15pt,90pt) {$\matU_1$};
    \node[anchor=center] at (-15pt,130pt) {$\matQ$};
    \node[anchor=center] at (75pt,90pt) {$\matH$};
    \node[anchor=center] at (30pt,170pt) {$\matU_2^{\top}$};
    \node[anchor=center] at (30pt,210pt) {$\matK^{\top}$};
    \node[anchor=center] at (115pt,170pt) {$\matV$};

    \begin{scope}[very thick, arrows=-stealth]
    \draw (25pt,40pt) -- (25pt,20pt);
    \draw (-15pt,80pt) -- (10pt,61pt);
    \draw (75pt,80pt) -- (40pt,61pt);
    \draw (-15pt,120pt) -- (-15pt,100pt);
    \draw (75pt,120pt) -- (75pt,100pt);
    \draw (30pt,160pt) -- (60pt,141pt);
    \draw (115pt,160pt) -- (90pt,141pt);
    \draw (30pt,200pt) -- (30pt,180pt);
    \end{scope}

    \begin{scope}[thick, gray, arrows=-stealth]
    \draw (-4pt,126pt) -- (1pt,91.5pt);
    \draw (4pt,126pt) -- (4pt,93pt);
    \draw (63pt,90pt) -- (27pt,52pt);
    \draw (4pt,90pt) -- (23pt,52pt);
    \draw (25pt,50pt) -- (10pt,13pt);
    \draw (-3pt,32pt) -- (4pt,13pt);
    \draw (7pt,32pt) -- (7pt,14pt);
    \draw (75pt,130pt) -- (66pt,91pt);
    \draw (53pt,110pt) -- (60pt,91pt);
    \draw (63pt,110pt) -- (63pt,92pt);
    \draw (100pt,170pt) -- (77pt,132pt);
    \draw (50pt,170pt) -- (73pt,132pt);
    \draw (50pt,206pt) -- (50pt,173pt);
    \draw (42pt,206pt) -- (47pt,171.5pt);
    \end{scope}

    \begin{scope}[thick, gray]
    \draw[fill=white] (-4pt,126pt) circle (0.6ex);
    \draw[fill=white] (4pt,126pt) circle (0.6ex);
    \draw[fill=white] (4pt,90pt) circle (0.6ex);
    \draw[fill=white] (63pt,90pt) circle (0.6ex);
    \draw[fill=white] (25pt,50pt) circle (0.6ex);
    \draw[fill=white] (7pt,12pt) circle (0.6ex);
    \draw[fill=white] (75pt,130pt) circle (0.6ex);
    \draw[fill=white] (100pt,170pt) circle (0.6ex);
    \draw[fill=white] (50pt,170pt) circle (0.6ex);
    \draw[fill=white] (50pt,206pt) circle (0.6ex);
    \draw[fill=white] (42pt,206pt) circle (0.6ex);
    \end{scope}

\end{tikzpicture}}
      \vspace{4pt}
      \caption{Sketch of the computational DAG for the approximate attention algorithm~\cite{alman2023fast}. Each box represents a set of nodes. In matrix multiplication, we first multiply pairs of entries from the input matrices, and then sum these up to from an entry in the output. In $\matU_1$ (and $\matU_2$), each node is computed from at most $g$ nodes of $\matQ$ (and $\matK$).}
        \label{fig:graph}
    \end{minipage}
    \hspace{.1\textwidth}
    \begin{minipage}{.44\textwidth}
      \centering
      \resizebox{0.9\textwidth}{!}{
      \begin{tikzpicture}

    \draw (0pt,0pt) -- (0pt,80pt) -- (60pt,80pt) -- (60pt,0pt) -- cycle;

    \draw (80pt,0pt) -- (80pt,80pt) -- (120pt,80pt) -- (120pt,0pt) -- cycle;

    \draw (80pt,100pt) -- (80pt,160pt) -- (120pt,160pt) -- (120pt,100pt) -- cycle;

	\node[anchor=center] at (30pt,40pt) {$\matU_1$};
    \node[anchor=center] at (100pt,130pt) {$\matH$};

    \begin{scope}[arrows=-stealth]
    \draw (63pt,40pt) -- (77pt,40pt);
    \draw (100pt,97pt) -- (100pt,83pt);
    \end{scope}

    \begin{scope}[very thick, gray, arrows=-stealth]
    \draw (28pt,67.5pt) -- (76pt,67.5pt);
    \draw (88.5pt,132pt) -- (88.5pt,84pt);
    \end{scope}

    \begin{scope}[ultra thick, gray]
    \draw (2pt,78pt) -- (2pt,57pt) -- (23pt,57pt) -- (23pt,78pt) -- cycle;
    \draw (82pt,78pt) -- (82pt,57pt) -- (95pt,57pt) -- (95pt,78pt) -- cycle;
    \draw (82pt,158pt) -- (82pt,137pt) -- (95pt,137pt) -- (95pt,158pt) -- cycle;
    \end{scope}

    \begin{scope}[thick, gray]
    \draw (-10pt,0pt) -- (-13pt,0pt) -- (-13pt,80pt) -- (-10pt,80pt);
    \draw (-13pt,40pt) -- (-16pt,40pt);
    \draw (0pt,90pt) -- (0pt,93pt) -- (60pt,93pt) -- (60pt,90pt);
    \draw (30pt,93pt) -- (30pt,96pt);
    \draw (70pt,100pt) -- (67pt,100pt) -- (67pt,160pt) -- (70pt,160pt);
    \draw (64pt,130pt) -- (67pt,130pt);
    \draw (80pt,170pt) -- (80pt,173pt) -- (120pt,173pt) -- (120pt,170pt);
    \draw (100pt,173pt) -- (100pt,176pt);
    \draw (99pt,78pt) -- (102pt,78pt) -- (102pt,57pt) -- (99pt,57pt);
    \draw (102pt,67.5pt) -- (105pt,67.5pt);
    \draw (82pt,54pt) -- (82pt,51pt) -- (95pt,51pt) -- (95pt,54pt);
    \draw (88.5pt,51pt) -- (88.5pt,48pt);
    \draw (2pt,54pt) -- (2pt,51pt) -- (23pt,51pt) -- (23pt,54pt);
    \draw (12.5pt,51pt) -- (12.5pt,48pt);
    \end{scope}

    \node[anchor=center] at (-21pt,40pt) {$n$};
    \node[anchor=center] at (59pt,130pt) {$r$};
    \node[anchor=center] at (30pt,100pt) {$r$};
    \node[anchor=center] at (100pt,181pt) {$d$};
    \node[anchor=center] at (110pt,67.5pt) {$t_1$};
    \node[anchor=center] at (88.5pt,42pt) {$t_2$};
    \node[anchor=center] at (12.5pt,42pt) {$t_3$};

\end{tikzpicture}}
      \vspace{4pt}
      \caption{Illustration of matrix multiplication with tiles, for the example $\matU_1 \matH$. The output matrix is split to tiles of size $t_1 \times t_2$, and each of these is aggregated in tiles of width $t_3$ (respectively, height $t_3$) in the corresponding row strip of $\matU_1$ (column strip of $\matH$). Altogether, the output matrix is computed via $\frac{n}{t_1} \cdot \frac{d}{t_2} \cdot \frac{r}{t_3}$ tiling steps.}
      \label{fig:matmul}
    \end{minipage}
\end{figure}

Intuitively speaking, this tiling is optimal for standard matrix multiplication because it allows to execute the highest number $\Theta(M^{3/2})$ of the scalar multiplications while the tile sizes in the input/ output matrices is still $\Theta(M)$. The best such tiling is naturally obtained when we use square-shaped tiles in the inputs and the output. However, with our special matrices $\matU_1$ and $\matU_2$, the situation is significantly different. The entries in a given row of $\matU_1$ are not inputs, but values computed from the corresponding row of $\matQ$; in fact, by loading only $w$ entries from the $i^{\text{th}}$ row of $\matQ$ (for some $w \leq d$), we can compute up to $\tau(w)=\binom{g+w}{g}$ different entries in the $i^{\text{th}}$ row of $\matU_1$.

This means that in order to form an essentially square-shaped tile in $\matU_1$ with $\Theta(M)$ load operations, we actually need to solve the following equation for $w$:
\begin{equation} \label{eq:w}
\tau(w) = \frac{\Theta(M)}{w}.
\end{equation}
Let us denote the solution of this by $w^*$. We can then select $\frac{\Theta(M)}{w^*}$ rows of $\matU_1$, and load $w^*$ values in the corresponding rows of $\matQ$, thus generating up to $\tau(w^*)$ entries in each row of $\matU_1$.

Forming tiles of this shape is the key idea behind our results. 
Intuitively, Case II marks the general case where this choice of $w^*$ is indeed possible. The other settings can be be understood as special cases. In Case I, we essentially have $w^* \geq d$; since we cannot load more than $d$ values in a row of $\matQ$, we can only form `taller' rectangular tiles. In Cases III and IV, we have $w^* \leq g$, but we still need to load $g$ values from $\matQ$ in many cases because many degree-$g$ terms in the polynomial actually contain $g$ different variables; hence we need to form `wider' rectangular tiles.

Note that obtaining a closed-form expression for $w^*$ from Equation~\eqref{eq:w} is non-trivial; instead, we rely on standard bounds on binomial coefficients. Specifically, we use the fact that $\left( \frac{a}{b} \right)^b \leq {a \choose b} \leq \left(\frac{e \cdot a}{b} \right)^b$.

\subsection{Generic upper bounds}

Note that if we naively apply the standard matrix multiplication method with tiles of size $\frac{\sqrt{M}}{4} \times \frac{\sqrt{M}}{4}$ on our computation, we get an algorithm of I/O cost $O(\frac{n \cdot r \cdot d}{\sqrt{M}})$. This is indeed viable when 
$\frac{\sqrt{M}}{4} \leq d$, i.e.\ each of the matrices can fit the square tiles. 
When 
$\frac{\sqrt{M}}{4} > d$, we can instead form tiles of size $\frac{M}{4d} \times d$ in both $\matH$ and the output matrix, and tiles of size $\frac{M}{4d} \times \frac{M}{4d}$ in $\matU_1$ (assuming that these tiles fit, i.e.\ we have $\frac{M}{4d} \leq r$). Each of these tiles only needs $\frac{M}{4}$ I/O operations, because in $\matU_1$, we can obtain all the entries of a row with only $d$ I/O steps. The number of tiling steps is altogether $\frac{n}{\frac{M}{4d}} \cdot \frac{r}{\frac{M}{4d}} = O(\frac{nrd^2}{M^2})$. With an I/O cost of $O(M)$ per tile, we altogether get $O(\frac{nrd^2}{M})$ I/O steps.

Note that $\frac{n \cdot r \cdot d}{\sqrt{M}} \leq \frac{nrd^2}{M}$ exactly when $\sqrt{M} \leq d$, so the two bounds above can be easily combined. This gives a generic upper bound that holds over Cases II, III and IV.

\begin{mylemma} \label{lem:static}
If $M = o (d \cdot r)$, the optimal I/O cost of Approximate Attention is upper bounded by
\[ O\left(
\min\left( \frac{n \cdot r \cdot d^2}{M}, \, \frac{n \cdot r \cdot d}{\sqrt{M}}\right) \right) \, . \]
\end{mylemma}

\subsection{Case I: proof of Theorem~\ref{th:case1}}

In Case I, we have $d \cdot \tau(d) = d \cdot r = O(M)$, i.e.\ there is a constant $0 \! < \! c_0 \! < \! 1$ such that $c_0 \cdot d \cdot r \leq \frac{1}{4} M$. This means that we can select $\frac{M}{4 \cdot d}$ rows from $\matU_1$, load all the $d$ values in the corresponding rows of $\matQ$ (at an I/O cost of $\frac{1}{4} M$), and generate all the $r$ entries in each of these rows in $\matU_1$.

For the computation, we can then split the output matrix into tiles of height $\frac{M}{4 \cdot d}$ and width $c_0 \cdot d$. We can compute each tile in a single iteration: we create the corresponding $\frac{M}{4 \cdot d}$ rows in $\matU_1$ as discussed, we load the tile of shape $r \times (c_0 \cdot d)$ from $\matH$ that needs to multiply it (at an input cost of at most $\frac{1}{4} M$), and then output the final result on the tile (output cost of at most $\frac{1}{4} M$). In  each tile, all the loaded entries in $\matQ$ and $\matH$ and the aggregated values in the output matrix can be kept in fast memory simultaneously (since $\frac{3}{4} M < M$) for the entire duration of computing the tile. On the other hand, the entries of $\matU_1$ are always recomputed from the entries in $\matQ$, without requiring any new I/O steps. The number of tiles needed to cover the output matrix is $\frac{n}{M / (4d)} \cdot \frac{d}{c_0 \cdot d} = O(\frac{n \cdot d}{M})$, altogether resulting in an I/O cost of $O(n \cdot d)$.

The first part of the computation, i.e.\ the other matrix multiplication $\matH= \matU_2^{\top}\matV$, can be executed in a similar way at the same I/O cost. Here $d$ entries in a row of $\matK$ allow us to generate the entire column of $\matU_2^{\top}$, so we can form tiles of shape $r \times \frac{M}{4 \cdot d}$ in $\matU_2^{\top}$, $\frac{M}{4 \cdot d} \times (c_0 \cdot d)$ in $\matV$ and $r \times (c_0 \cdot d)$ in $\matH$, aggregating each tile in the output matrix in $\frac{n}{M / (4d)}$ steps.

Note that the corresponding lower bound is straightforward: each entry of the input matrices $\matQ$, $\matK$ and $\matV$ must be loaded at least once, 
giving an I/O cost of at least $3 \cdot n \cdot d$.

\subsection{Case II: when $g \leq w^* < d$} \label{sec:case2_intuition}

In Case II, the solution to Equation~\eqref{eq:w} satisfies $g \leq w^* < d$. The best strategy here, as discussed above, is to tile $\matU_1$ by loading $w=w^*$ values from $\frac{\Theta(M)}{w}$ distinct rows of $\matQ$. Consider a specific row $\vecq$ of $\matQ$, and let $\vecu$ be the corresponding row in $\matU_1$. Hypothetically, if every set of $w$ loaded values from $\vecq$ could generate $\tau(w)$ \emph{distinct} entries in $\vecu$, then the strategy would provide matching upper and lower bounds. In this case, the corresponding tiles in $\matH$ and $\matU_1 \matH$ could both have height $\frac{\Theta(M)}{w}$ and width $\Theta(w)$, so all inputs and outputs in a tile would fit into fast memory. This altogether gives $\Theta\left(\frac{n}{\frac{M}{w}} \cdot \frac{d}{w}\right)$ tiles in the output matrix, each aggregated from $\Theta(\frac{r}{\frac{M}{w}})$ separate column strips in $\matU_1$ and row strips in $\matH$. The total number of tiling steps would be $\Theta(\frac{nd}{M} \cdot \frac{rw}{M}) = \Theta(\frac{ndrw}{M^2})$. With $\Theta(M)$ I/O steps per tile, the total cost is $\Theta(\frac{ndrw}{M})$. In Appendix~\ref{app:keylemma}, we show that no I/O strategy can be more efficient than this. The proof analyzes the so-called $S$-partitions of the computational graph, which is a long-established tool to derive I/O lower bounds on different computations~\cite{hong1981complexity,papp2025impact}.

In contrast to this hypothetical setting, the $w$ loaded values \emph{cannot} generate fully-disjoint sets of entries from $\vecu$ in each tile. To ensure that every entry of $\vecu$ is obtainable in at least one tile, the entries of $\vecq$ actually need to be loaded in multiple different tiles. This means that there will be many entries in $\vecu$ that can be generated in multiple tiles. We can assign each such entry of $\vecu$ to an arbitrary one of these tiles, but this still implies that the width of the average tile in $\matU_1$ will be lower than $\tau(w)$.

As such, we need to select different subsets of $w$ values each from $\vecq$, and assign each entry in $\vecu$ to one of these subsets that can generate it. For this, we divide the $d$ entries of $\vecq$ into $\frac{d \cdot g}{w}$ groups of size $\frac{w}{g}$ each, and consider all the ${\frac{d \cdot g}{w} \choose g}$ possible combinations of $g$ different groups as generator sets for a tile. All entries in $\vecu$ can be obtained in at least one combination, since they are generated from at most $g$ entries in $\vecq$. For entries in $\vecu$ that can be generated in multiple tiles, we simply assign it to the first such tile, hence some of our tiles in $\matU_1$ will be wider than others.

Altogether, each tile of the output matrix is aggregated through ${\frac{d \cdot g}{w} \choose g}$ steps. With tiles of size $\frac{\Theta(M)}{w} \times \Theta(w)$ in an output matrix of dimensions $n \times d$, and an I/O cost of $O(M)$ per tile, this gives a total I/O cost of $\frac{n}{\frac{\Theta(M)}{w}} \cdot \frac{d}{\Theta(w)} \cdot {\frac{d \cdot g}{w} \choose g} \cdot O(M) = O\left( n \cdot d \cdot {\frac{dg}{w} \choose g} \right)$.

Formally, the observations above can be expressed through the following lemma.

\begin{keylemma} \label{th:with_w}
Let us choose an integer $w$ such that $g \leq w \leq d$ and $w \cdot \tau(w) \leq \frac{1}{4} M$. Then the optimal I/O cost of Approximate Attention in Case II is upper bounded by
\[ O\left( n \cdot d \cdot {\frac{dg}{w} \choose g} \right) \, , \]
and lower bounded by
\[ \Omega\left( \frac{n \cdot r \cdot d  \cdot w}{M} \right) \, . \]
\end{keylemma}

The proof is provided in Appendix~\ref{app:keylemma}. Theorem~\ref{th:case2} then follows from this more general statement.

\renewcommand{\proofname}{Proof of Theorem~\ref{th:case2}}
\begin{proof}
Given Key Lemma~\ref{th:with_w}, the bounds in Theorem~\ref{th:case2} are obtained by choosing
\[ w_0 = \frac{1}{4 \cdot e} \cdot g \cdot M^{\frac{1}{g+1}} \]
for $w$. Note that this $w_0$ is only slightly different from the real $w^*$, due to only approximating the binomial coefficient in $\tau(w)$. First note that $w_0 \geq g$; this is equivalent to $M \geq (4e)^{g+1}$, which indeed holds if $g=o(\log M)$.
 We then show that $w_0$ satisfies $w_0 \cdot \tau(w_0) \leq \frac{1}{4} M$. Indeed, we have
\[ w_0 \cdot \tau(w_0) = w_0 \cdot {w_0+g \choose g} \leq w_0 \cdot \left(\frac{2 \cdot e \cdot w_0}{g}\right)^g \]using the upper bound on the binomial coefficient, and $w_0 \geq g$. Substituting $w_0$ into this, we get $\frac{1}{4e} \cdot \frac{g}{2^g} \cdot M$. This is indeed smaller than $\frac{1}{4} M$, since $g < 2^{g}$ for any positive integer $g$. From this, it also follows that $w_0 \leq d$, since in Case II, we have $d \cdot \tau(d) = \omega(M)$. 

Using $w=w_0$ in the lower bound of Key Lemma~\ref{th:with_w} directly gives our lower bound in Theorem~\ref{th:case2}:
\[ \Omega\left( \frac{n \cdot r \cdot d}{M} \cdot w_0 \right) = \Omega\left(\frac{n \cdot r \cdot d \cdot g}{M^{\frac{g}{g+1}}}\right) \, . \]
Similarly, the upper bound of Key Lemma~\ref{th:with_w} provides the upper bound in Theorem~\ref{th:case2} if we substitute $w=w_0$, and use ${\frac{dg}{w} \choose g} \leq (\frac{edg}{wg})^g$ and $r = {d+g \choose g} \geq (\frac{d}{g})^g$:
\[ O\left( n _{\!} \cdot _{\!} d _{\!} \cdot _{\!} {\frac{dg}{w_0} \choose g} \right) = O\left( n _{\!} \cdot _{\!} d _{\!} \cdot _{\!} \frac{e^g _{\!} \cdot _{\!} d^g}{w_0^g} \right) =
O\left( \frac{n _{\!} \cdot _{\!} d _{\!} \cdot _{\!} (4e^2)^g _{\!} \cdot _{\!} d^g}{g^g _{\!} \cdot _{\!} M^{\frac{g}{g+1}}} \right) = O\left(\frac{n _{\!} \cdot _{\!} r _{\!} \cdot _{\!} d _{\!} \cdot _{\!} (4e^2)^g}{M^{\frac{g}{g+1}}}\right) \, . \qedhere \]
\end{proof}

Finally, note that the discussion of Case II above implicitly assumes $g \leq d$. Otherwise (or in fact, already if $d=O(g)$), then the conditions for Case II cannot hold at all. Indeed, $d=O(g)$ means $d \cdot r \leq O(g)  \cdot \left( \frac{O(g)}{g} \right)^{g} \leq \delta^g$ for some constant $\delta$. As such, $M=o(d \cdot r)$ implies $M < \delta^g$, i.e.\ $g = \Omega(\log{M})$, which contradicts $g = o(\log{M})$. So if $d=O(g)$, then Case II is in fact an empty regime: we are always in one of the remaining three cases. However, note that for realistic choices of the parameters, $g \leq d$ typically holds, and hence Case II is indeed non-empty.

\subsection{Case III: when $w^* \leq g$}

In Case III, the optimal $w^*$ for square tiles would be smaller than $g$. However, many of the entries in $\matU_1$ still require at least $g$ values to generate. The best upper bounds here are mostly obtained simply via the generic bounds in Lemma~\ref{lem:static}; this is also what we find in Theorem~\ref{th:case3}. Recall from the introduction that for $g=O(1)$, this is tight to the lower bound up to a constant factor. Note that even for $g=\omega(1)$, the bounds are also tight in some cases: for instance, if we have $d=O(g)$, then $\frac{n \cdot r \cdot d^2}{M}$ again matches the lower bound up to a constant factor. On the other hand, when $d=\Omega(g^2)$, one can show that the gap between the bounds in Theorem~\ref{th:case3} is at least $g=\omega(1)$.

The regime of $d$ between $O(g)$ and $\Omega(g^2)$ is more challenging. With some further work, we can actually extend the tight upper bounds to $d$ being almost quadratic in $g$. The proof of this is more technical; it essentially adapts the upper bound idea from Case II with a choice of $w=g$, and some further adjustments. It also requires the extra assumption that $M$ is polynomial in $g$.

\begin{mylemma} \label{lem:specialupper}
In Case III, if we also assume that $d=O(g^{2-\delta_1})$ and $M = O(g^{\delta_2})$ for some constants $\delta_1, \delta_2 >0$, then the optimal I/O complexity of Approximate Attention is upper bounded by $O(\frac{n \cdot r \cdot d \cdot g}{M})$.
\end{mylemma}

The lower bound in Theorem~\ref{th:case3} uses the same tools ($S$-partitions) as in Case II, but requires a more complex proof. Intuitively, the terms that can be generated from less than $g$ variables have a much larger role here. Due to this, we add an extra assumption here that $d \geq 5g$, which is indeed realistic if e.g.\ $g=O(1)$. With $d \geq 5g$, a technical lemma shows that at least a constant fraction of the columns in $\matU_1$ use at least $\frac{g}{2}$ different variables. We can then restrict our analysis to this sub-matrix where any term requires loading $\frac{g}{2}$ values in $\matQ$, and use a similar proof to Case II. 

The formal proofs of Lemma~\ref{lem:specialupper} and the lower bound in Theorem~\ref{th:case3} are available in Appendix~\ref{app:case3}.

\subsection{Case IV}

Finally, for tiny cache sizes $M = O(g^2)$, we actually have another special case. Intuitively, here it becomes more beneficial to store the values of $\matU_1$ than to recompute them; this provides tight bounds of $\Theta\left(\frac{n \cdot d \cdot r}{\sqrt{M}}\right)$, similarly to standard matrix multiplication. These proofs are discussed in Appendix~\ref{app:case4}.

\section{Discussion and Conclusion}
In this work we studied the I/O complexity of Approximate Attention. We provided sharp upper and lower bounds for different parameter regimes, and showed that this approach can have significantly lower I/O cost than state-of-the-art methods like FlashAttention.

We note that our upper bounds, as well as the lower bound of Case I, are algorithm-agnostic: they hold for the approximate attention problem in general. In contrast to this, the lower bounds in Cases II-IV are inherently specific to the algorithm of Alman and Song.

Our proof techniques may also prove useful in future works for obtaining I/O bounds in other settings with special kinds of matrix multiplication. For instance, consider a sparse attention approach where the attention matrix is guaranteed to have at most $O(d)$ non-zero entries in each row: we can derive upper bounds for this in a very similar way to Lemma~\ref{lem:static}. As for the lower bound proofs for Cases II-IV, one natural candidate for generalization here is the recent method of Aliakbarpour et al~\cite{aliakbarpour26support}: this presents an extension of the Alman-Song algorithm, where matrices are decomposed into a dense part (handled via polynomial approximation) and a sparse part with larger entries. In any such decomposition, if the computations in the sparse part have I/O cost below $O(n \cdot d)$ due to their sparsity, then the I/O analysis reduces to our bounds. More generally, while our lower bound proofs are more specific to this concrete setting, their high-level ideas may still carry over to other algorithms for attention or matrix multiplication variants.

While the theoretical bounds are promising to improve the I/O cost of attention algorithms, it is well-known that the theoretically-fastest algorithms are not always the best-performing ones in practice. For example, the combinatorial nature of the parameter $r$ (which originates from the analysis of \cite{alman2023fast}) can be prohibitively large for some practical parameter regimes. It is an interesting direction for future work to also investigate the efficiency and applicability of our algorithms in practice.

\bibliographystyle{plain}
\bibliography{references}

\appendix

\section{Proof of Key Lemma~\ref{th:with_w}} \label{app:keylemma}

\subsection{Upper bound proof}
The upper bound proof follows the idea outlined in Section~\ref{sec:case2_intuition}. We split the output matrix into tiles of height $\frac{M}{4w}$ and width $w$. This already specifies a row strip of size $\frac{M}{4w} \times r$ in $\matU_1$ and a column strip of size $r \times w$ in $\matH$ that are needed to compute this output tile. We aggregate the output tile in ${\frac{d \cdot g}{w} \choose g}$ steps, splitting the row strip of $\matU_1$ horizontally (and the column strip of $\matH$ vertically) into aggregation tiles.

The aggregation tiles are obtained as follows. We split the column indices $\{1, \ldots, d\}$ of $\matQ$ into groups of size $\frac{w}{g}$: let $G_1=\{1, \ldots, \frac{w}{g}\}$, $G_2=\{\frac{w}{g}\!+\!1, \ldots, 2 \! \cdot \! \frac{w}{g}\}$, $\ldots$, $G_{\frac{dg}{w}}=\{d\!-\!\frac{w}{g}\!+\!1, \ldots, d\}$. Each aggregation tile is formed by selecting $g$ distinct groups, i.e.\ the number of aggregation tiles is ${\frac{dg}{w} \choose g}$. In each aggregation tile, we consider the union of the $g$ groups combined, and consider the resulting $w$ distinct column indices; specifically, for each of the corresponding $\frac{M}{4w}$ rows, we load the values from these $w$ columns of $\matQ$. This corresponds to $\frac{1}{4} M$ load operations for each aggregation tile.

Note that some of the terms in the polynomial require an entry from $g$ distinct groups; these columns of $\matU_1$ can only be generated on the specific tile that combines these $g$ groups. Since any term contains at most $g$ variables, each entry of $\matU_1$ can be generated in at least one tile. However, there will also be numerous terms that can be generated in multiple different tiles: e.g.\ a term that only uses variable indices in $G_1$ and $G_2$ can be assigned to any of the tiles that combine $G_1$, $G_2$ and a selection of $(g-2)$ arbitrary further groups. For simplicity, we will always assign each term to the first tile that can generate it. That is, we iterate through the group subsets via a lexicographic order of their sorted indices, and we assign each term to the first tile that can generate it in this order. Note that for any term, it is straightforward to determine the tile that generates it: we consider the groups that contain at least one variable of this term, and if the number of groups is smaller than $g$, we use the smallest non-included group indices for the rest of the groups. As such, for any tile and term, it is easy to decide whether the term is assigned to this given tile or not. This means that for any given tile, we can iterate over all the $\tau(w)$ terms that can be generated from the corresponding $w$ inputs, and generate only the terms that belong to this specific tile, avoiding duplications. With this method, the width of each tile in $\matU_1$ will be at most $\tau(w)$ (the maximum number of terms we can generate from $w$ variables), and at least $(\frac{w}{g})^g$ (the terms using $g$ distinct groups that can only be obtained on this tile).

We point out that in Case II, the aggregation tiles we form are always taller than wider. Indeed the height of the tiles is $\frac{M}{4w}$, and their width is at most $\tau(w)$. We have $\frac{M}{4w} \geq \tau(w)$ due to $w \cdot \tau(w) \leq \frac{1}{4} M$. This is crucial to ensure that the corresponding tiles formed in $\matH$ indeed fit into fast memory: since their height is at most $\tau(w) \leq \frac{M}{4w}$, and their width is $w$, they contain at most $\frac{1}{4} M$ entries.

Altogether, each aggregation tile requires at most $\frac{1}{4} M$ loaded inputs from both $\matU_1$ and $\matH$, which are always kept in fast memory. The results are aggregated in a tile of size $\frac{1}{4} M$ in the output matrix, which are written to slow memory in the end. As such, each sub-computation fits into the working memory of $M$, and has I/O cost of $O(M)$. The number of tiles in the output matrix is $\frac{n}{\frac{M}{4w}} \cdot \frac{d}{w}$, and each are aggregated in ${\frac{dg}{w} \choose g}$ steps. This results in a total I/O cost of \[ \frac{4nw}{M} \cdot \frac{d}{w} \cdot {\frac{dg}{w} \choose g} \cdot M = O\left(n \cdot d \cdot {\frac{dg}{w} \choose g}\right) \, . \]

We note that in the red-blue pebbling model, generating the entries in $\matU_1$ requires no further working memory; the process of iterating through the entries is incorporated into the pebbling strategy. However, in a practical implementation, this requires a few further fast memory entries, e.g.\ index variables to iterate through the $\tau(w)$ possible combinations.

We also note that when aggregating the results in a tile of size $\frac{1}{4} M$ above, this actually incurs an I/O cost of $\frac{1}{2} M$, i.e.\ two I/O steps for each entry of the matrix (in the appropriate red-blue pebbling variant; see Appendix~\ref{app:RBPvariants} for a discussion). Indeed, we first need to load the current aggregated value for the entry in the output matrix, add the newly computed terms to this value, and then save the new (partial) aggregated value to slow memory again. However, this is simply a clarification of a technical detail; since this cost is still in $O(M)$, it has no effect on the analysis of I/O costs above.

Extending the algorithm to the entire computation $\matU_1 (\matU_2^{\top}\matV)$ is straightforward. We simply execute the two multiplications separately, considering $\matH=\matU_2^{\top}\matV$ as an output in the first and as an input in the second. In the first computation, we use the same tiling strategy as discussed above, just with the role of the matrices exchanged. That is, we form column strips of width $\frac{M}{4w}$ in $\matU_2^{\top}$ and the corresponding row strips of height $\frac{M}{4w}$ in $\matV$. We split each of these strips to ${\frac{dg}{w} \choose g}$ tiles (of differing height) in $\matU_2^{\top}$. This results in tiles of width $w$ and height at most $\tau(w)$ in $\matH$. Each such tile will be aggregated in $\frac{n}{\frac{M}{4w}}$ steps from the different column strips in $\matU_2^{\top}$ and row strips in $\matV$. With ${\frac{dg}{w} \choose g} \cdot \frac{d}{w}$ tiles in $\matH$, the total I/O cost is the same as in the second matrix multiplication.

We point out that in almost all cases, the upper bound for Theorem~\ref{th:case2} (obtained via Key Lemma~\ref{th:with_w}) is tighter than the generic upper bound from Lemma~\ref{lem:static}. In particular, when $M = O(d^2)$ and hence the generic upper bound is $O(\frac{n \cdot r \cdot d}{\sqrt{M}})$, then our upper bound in Theorem~\ref{th:case2} is smaller when $M^{\frac{1}{2}-\frac{1}{g+1}} \geq (4e^2)^g$ holds. For any $g \geq 2$, the left side here is lower bounded by $M^{\frac{1}{6}}$, and then the claim indeed follows from $g = o(\log{M})$. On the other hand, when $M = \omega(d^2)$, and hence the generic upper bound is $O(\frac{n \cdot r \cdot d^2}{M})$, then our upper bound is smaller when $d \geq M^{\frac{1}{g+1}} \cdot (4e^2)^g$. This also holds if we have any constant $\delta>0$ such that $M \leq d^\delta$. Indeed, then $(4e^2)^g$ can be upper bounded by $M^{\frac{\delta}{2}}$ (because $g = o(\log{M})$) and $M^{\frac{1}{g+1}}$ can also be upper bounded by $M^{\frac{\delta}{2}}$ asymptotically (assuming that $g=\omega(1)$; otherwise, the upper and lower bounds are tight anyway). The generic upper bound may only be superior for very specific parameter combinations, e.g.\ if we have $M = \left( \frac{d}{g} \right)^{g+1}$.

\subsection{Lower bound proof}

The lower bound proof corresponds to the hypothetical case above where each set of $w$ entries from $\vecq$ generates $\tau(w)$ distinct entries in $\vecu$, and hence it is likely somewhat lower than the actual optimum.

When proving lower bounds for I/O cost, one of the most widely used tools are so-called $S$-partitions. Defined by Hong and Kung in their original paper on red-blue pebbling and I/O complexity~\cite{hong1981complexity}, an $S$-partition is a way to partition the node of the computational graph such that, intuitively, each class can be computed with only $2S$ I/O steps. 

\begin{definition}
An \emph{$S$-partition} of a computational DAG (for some integer parameter $S$) is a disjoint partitioning $D_1$, ..., $D_k$ of the vertices of the DAG such that
\begin{itemize}[leftmargin=1.6em, topsep=1pt, itemsep=4pt]
 \item the classes $D_1$, ..., $D_k$ form a topological order of the computation, i.e.\ there are no edges from $D_j$ to $D_i$ with $i<j$;
 \item for each class $D_i$, there is a so-called \emph{dominator set} of size $S$ in the DAG, i.e.\ a set of nodes such that every path from a source node to a node in $D_i$ contains a node in this set;
 \item for each class $D_i$, the so-called \emph{minimum set} of $D_i$, i.e.\ the set of nodes in $D_i$ without a child in $D_i$, has size at most $S$.
\end{itemize}
\end{definition}
For more details on $S$-partitions, we refer the reader to~\cite{hong1981complexity}. The crucial property of $S$-partitions is that they directly allow to lower bound the optimal I/O cost when using $S=2M$.

\begin{lemma}[From Hong \& Kung~\cite{hong1981complexity}] \label{lem_spart}
If $\textsc{min}_{2M}$ denotes the minimal number of classes in any $2M$-partition of a computational DAG, then the optimal I/O cost of this computation is at least $M \cdot (\textsc{min}_{2M} - 1)$.
\end{lemma}

In the lower bound proof, we analyze the number of classes that each $2M$-partition must have in our computation. Through the lemma above, this directly yields the lower bound in Key Lemma~\ref{th:with_w}.

\renewcommand{\proofname}{Proof of Key Lemma~\ref{th:with_w}, lower bound}

\begin{proof}
Consider first the computational DAG restricted to the multiplication $\matU_1 \matH$, i.e.\ where the entries of $\matH$ are all source nodes. Let us refer to the nodes that correspond to (the output of) multiplying an entry of $\matU_1$ and an entry of $\matH$ as \emph{internal nodes}. The DAG altogether has $n \cdot r \cdot d$ internal nodes. Similarly to previous I/O lower bounds on matrix multiplication variants, the main idea of the proof is to upper bound the number of internal nodes that can be contained in each partition, thus obtaining a lower bound on $\textsc{min}_{2M}$.

Let $D_i$ be any partition in a $2M$-partition. Consider the minimum set of $D_i$ and a dominator set of size at most $2M$ for $D_i$. This is altogether at most $4M$ nodes; even if they are all internal, this gives at most $4M$ internal nodes. Let $\Gamma$ be the set of all other internal nodes contained in $D_i$; let us call these \emph{uncovered} internal nodes. For every uncovered node $v \in \Gamma$, we must have that (i) on each directed path from a source node to $v$, we have a node in the dominator set, and (ii) at least one successors of $v$ is in the minimum set. Note that each internal node has exactly one parent in $\matH$ that is a source node, so for uncovered nodes, these must all be in the dominator set. An uncovered node $v \in \Gamma$ also has a parent in $\matU_1$; however, here there are two options, either the parent in $\matU_1$ is in the dominator set, or all the parents of the parent (located in $\matQ$) are in the dominator set. As for the successor of $v$ in the minimum set, this may be the sink node in the output matrix that corresponds to $v$, or any node in the summation tree from $v$ to the sink.

For any $i \in [n]$, let $w_i\,\!^{(\matQ)}$ and $w_i\,\!^{(\matU_1)}$, respectively, denote the number of entries in the $i^{\text{th}}$ row of $\matQ$ and the $i^{\text{th}}$ row of $\matU_1$ that are in the dominator set. Let $w_i = w_i\,\!^{(\matQ)} + w_i\,\!^{(\matU_1)}$. Let us consider the number of nodes in the $i^{\text{th}}$ row of $\matU_1$ that can have a child in $\Gamma$. These nodes must either be in the dominator set, or have all their parents in the dominator set, so their number is at most $\tau(w_i\,\!^{(\matQ)}) + w_i\,\!^{(\matU_1)}$. With $w_i = w_i\,\!^{(\matQ)} + w_i\,\!^{(\matU_1)}$, this is at most $\tau(w_i)$.

Let us sort $\matU_1$ into two sub-matrices: let $\matU_1^{(a)}$ contain those rows where $w_i \geq w$, and $\matU_1^{(b)}$ contain those rows where $w_i < w$. The matrix multiplication can similarly be split into two sub-multiplications $\matU_1^{(a)} \matH$ and $\matU_1^{(b)} \matH$. In $\matU_1^{(a)}$, each row has at least $w$ nodes in the dominator set, so the number of rows is at most $\frac{2M}{w}$. On the other hand, the number of nodes of $\matH$ in the dominator set is also at most $2M$. Each row in $\matU_1^{(a)}$ has only one value that is multiplied with any of the $2M$ nodes in $\matH$, hence the number of uncovered internal nodes in $\matU_1^{(a)} \matH$ is at most $\frac{2M}{w} \cdot 2M = \frac{4M^2}{w}$.

On the other hand, in $\matU_1^{(b)}$, each row can have at most $\tau(w) \leq \frac{M}{4w}$ nodes that have a child in $\Gamma$. Recall that each node in $\Gamma$ has a successor in the minimum set; this means that there are at most $2M$ sink nodes in the output matrix such that each node in $\Gamma$ is a predecessor of one of these sink nodes. For such a sink node $v_0$, we can have at most $\frac{M}{4w}$ nodes in $\Gamma$ that are predecessors of $v_0$, because at most $\frac{M}{4w}$ nodes of $\matU_1$ can have a child in $\Gamma$ in the given row. This limits the number of uncovered internal nodes in $\matU_1^{(b)} \matH$ to $\frac{M}{4w} \cdot 2M = \frac{M^2}{2w}$.

Altogether, the number of nodes in $\Gamma$ is at most $4M+\frac{4M^2}{w}+\frac{M^2}{2w} = O(\frac{M^2}{w})$. With $n \cdot r \cdot d$ internal nodes, this means that the number of partitions is at least $\Omega(\frac{n \cdot r \cdot d \cdot w}{M^2})$. Using Lemma~\ref{lem_spart}, we get that the optimal I/O cost is at least $\Omega(2 \cdot \frac{n \cdot r \cdot d \cdot w}{M^2} - 1) \cdot M = \Omega(\frac{n \cdot r \cdot d \cdot w}{M})$. This finishes the proof of the lower bound.

In order to extend the proof from the multiplication $\matU_1 \matH$ to the entire computation $\matU_1 (\matU_2^\top \matV)$, we need to consider the fact that the nodes in $\matH=\matU_2^\top \matV$ are in fact not sources. That is, instead of having an entry of $\matH$ in the dominator set, we could include some of its predecessors. However, each node in $\matH$ has $n$ parents in the same column of $\matV$, which would all need to be in the dominator set in this case. As such, if a dominator set in our previous proof contained $x \in \{0, ..., r\}$ entries from a specific column of $\matH$, we could only replace these by $n$ nodes in $\matV$ (and further nodes in $\matU_2$ or $\matK$, but these can be ignored now). Intuitively, with $n > r$, this only increases the size of the dominator set. As such, for the whole computation $\matU_1 (\matU_2^\top \matV)$, it still holds that there can be at most $2M$ nodes in $\matH$ with a child that belongs to $\Gamma$, which allows us to bound the number of uncovered internal nodes in $\matU_1^{(b)} \matH$ exactly as before.
\end{proof}

\subsection{Discussion on red-blue pebbling variants} \label{app:RBPvariants}

We note that the standard version of the red-blue pebble game raises some modeling questions regarding operators with many inputs, such as the summation part in the matrix multiplication. If the summation is modeled through a single node, then this requires all inputs in fast memory at the same time, which means that there is no viable pebbling strategy at all if $M \leq r$.

Previous works have often avoided the discussion of this issue by forming a binary `summation tree' in the computational DAG, and assuming that aggregation happens this way~\cite{saha2024complexity}. This is in fact a strong restriction on the execution strategies considered. However, for heavily symmetric operations like matrix multiplication, this had no effect on the optimal I/O complexity: since the optimal strategy was to form tiles of specific sizes, the tiles can naturally consist of neighboring rows/columns in the matrices, and hence the summation within a tile essentially corresponds to a sub-tree in the summation tree of the final aggregated value.

The same summation tree approach becomes somewhat more problematic for our Approximate Attention problem. Our results indicate that for an I/O-efficient execution, we need to group together specific columns of the matrix $\matU_1$; hence it becomes critical that the corresponding group of nodes also form a sub-tree in the summation tree of the given entry of the output matrix. As such, this requires a model that inherently entangles the DAG representation of the computation with the optimal I/O strategy to execute it.

In general, a more appropriate approach to resolve the question of summation trees is to consider a generalization of the red-blue pebble game with partial computations~\cite{papp2025impact,sobczyk2026complexity}. This generalized model allows to keep the clean representation of the aggregations as a single node with $r$ incoming edges (as in Figure~\ref{fig:graph}), requiring no summation trees at all. The aggregation process is then modeled by the pebble game itself, which allows to store partially computed results in fast or slow memory. This extended model is a much more accurate description of how the computation is executed in practice. Our algorithm descriptions in the upper bounds are also easiest to interpret in this model.

On the other hand, deriving I/O lower bounds for this generalized pebble game also requires a small adjustment to the $S$-partitioning concept, using so-called $S$-edge-partitions~\cite{papp2025impact}. While our lower bound proofs are presented in terms of the standard red-blue pebble game and $S$-partitions, they also carry over to this generalized model and $S$-edge-partitions by simply shifting the focus from the internal nodes to their (single) outgoing edges, similarly to the proofs adaptations for classic computations~\cite{papp2025impact}. 

\section{Proof of Theorem~\ref{th:case3}} \label{app:case3}

\subsection{Upper bound discussion}

The upper bound in Theorem~\ref{th:case3} simply re-states the generic upper bounds from Lemma~\ref{lem:static}. Note that once again when $g=O(1)$, the bounds are tight up to a constant factor. For instance, the factor of difference between $\frac{n \cdot r \cdot d}{\sqrt{M}}$ and $\frac{n \cdot r \cdot d \cdot g}{M}$ is $\frac{\sqrt{M}}{g}$. Due to $g=\Omega(\log M)$, we have $M = 2^{O(g)}$; for $g=O(1)$, this implies $M=O(1)$, and thus $\frac{\sqrt{M}}{g}=O(1)$.

We point out that besides $g=O(1)$, the bounds are also tight for some other parameter configurations. That is, let us assume $g=\omega(1)$. In this case, for instance when we have $d=O(g)$, then $\frac{n \cdot r \cdot d^2}{M}$ is again only a constant factor away from the lower bound. On the other hand, when $d=\Omega(g^2)$, the gap between the upper and lower bound is $\frac{\min(d, \,\sqrt{M})}{g}$, which is at least $g$ (since $M=\Omega(g^2)$). As such, this case exhibits a larger gap than a constant factor.

At this point, it is a natural question how the upper bounds behave when $d$ is between $g$ and $g^2$ (with $g=\omega(1)$). As a further contribution, we present another algorithm that is very similar to that of Key Lemma~\ref{th:with_w}, but adapted to Case III, in order to answer this question. The upper bound obtained by this algorithm will be weaker than the generic bounds of Lemma~\ref{lem:static} when $d=\Omega(g^2)$. However, for the case when there are constants $\delta_1, \delta_2 >0$ such that $d=O(g^{2-\delta_1})$ and $M = O(g^{\delta_2})$, this algorithm will once again provide a tight upper bound of $O(\frac{n \cdot r \cdot d \cdot g}{M})$; this is formally stated in Lemma~\ref{lem:specialupper}. This result extends the tight upper bound from $d=O(g)$ to $d$ almost quadratic in $g$, assuming that $M$ is at most polynomial in $g$.

\subsection{Detour: proof of Lemma~\ref{lem:specialupper}}

More specifically, our algorithm for Lemma~\ref{lem:specialupper} proves an upper bound of
\begin{equation} \label{eq:maxbound}
O \left( n \cdot d \cdot \left( {d \choose g} + \frac{4 \cdot r \cdot g}{M} \right) \right)
\end{equation}
when $d=O(g^{2-\delta_1})$ and $M = O(g^{\delta_2})$ for some constant $\delta >0$.

We first show that in our parameter regime, we have ${d \choose g} \leq \frac{4 \cdot r \cdot g}{M}$; this implies that the formula in Equation~\eqref{eq:maxbound} can be upper bounded by $\frac{8 \cdot n \cdot r \cdot d \cdot g}{M}$, which is indeed a constant factor away from the lower bound of Theorem~\ref{th:case3}.

The claim we need to prove is equivalent to
\[ \frac{M}{4g} \leq \frac{r}{{d \choose g}} = \frac{{d+g \choose g}}{{d \choose g}} = \frac{(d+g) \cdot \ldots \cdot (d+1)}{d \cdot \ldots \cdot (d-g+1)} \, . \]
The expression on the right-hand side can be lower bounded by $(\frac{d+g}{d})^g$. With $d=O(g^{2-\delta_1})$, this can be further lower bounded by $\left(\frac{g^{2-\delta_1'}+g}{g^{2-\delta_1'}}\right)^g$ for some constant $\delta_1'$ such that $\delta_1 > \delta_1' > 0$. This is then
\[  \left(\frac{g^{1-\delta_1'}+1}{g^{1-\delta_1'}}\right)^{g} = \left(1 + \frac{1}{g^{1-\delta_1'}}\right)^{g} =
\left(\left(1 + \frac{1}{g^{1-\delta_1'}}\right)^{(g^{1-\delta_1'})}\right)^{(g^{\delta_1'})} \, .\]
The outer parenthesis goes to $e$ as $g$ goes to infinity, hence from some point we can lower bound it by $\frac{e}{2}$. We also have $M  = O(g^{\delta_2})\leq g^{\delta_2'}$ for some $\delta_2' > \delta_2$. Hence we only need to show $g^{\delta_2'} \leq 4 \cdot g \cdot \left( \frac{e}{2} \right)^{(g^{\delta_1'})}$; this indeed holds asymptotically for any constants $\delta_1'$, $\delta_2'$, since the left side is polynomial, while the right side is exponential in $g$.


It remains to show the I/O strategy that proves the upper bound in Equation~\eqref{eq:maxbound}.
This is very similar to the upper bound proof of Key Lemma~\ref{th:with_w}, with a specific choice of $w=g$. That is, we split the output matrix into tiles of height $\frac{M}{4g}$ and width $g$. The row strip of size $\frac{M}{4g} \times r$ in $\matU_1$ is split into ${d \choose g}$ aggregation tiles horizontally, and the column strips of width $g$ in $\matH$ are similarly split vertically.

For the aggregation tiles, now each subset of $\{ 1, \ldots, d\}$ of size $g$ will form a separate tile; this corresponds to having $|G_1|=\ldots=|G_d|=1$ in the key lemma proof. In each aggregation tile, we load the corresponding $g$ values; this gives one term that can only be generated on this tile, and other terms that can be generated in multiple tiles. We once again assign each of the latter terms to the first possible tile in the lexicographic order. Similarly to the analysis before, this approach results in $\frac{n}{\frac{M}{4g}} \cdot \frac{d}{g}$ tiles in the output matrix.

The key difference to Key Lemma~\ref{th:with_w} here is that the width of the tiles formed in $\matU_1$ might be larger than their height: that is, we may have $\tau(g) > \frac{M}{4g}$. In this case, the corresponding tile in $\matH$, which has size $\tau(g) \times g$, might not fit into fast memory. Due to this, we split each of these tiles further into sub-tiles of height at most $\frac{M}{4g}$; with this, the corresponding tile in $\matH$ has size $\frac{M}{4g} \cdot g = \frac{1}{4} M$. Executing this split only requires us to memorize the current position in the aggregation tile in $\matU_1$. For each sub-tile, the same $\frac{M}{4g} \times g$ values can be kept from $\matQ$, the tile aggregated in the output can also remain in fast memory. However, from $\matH$, we need to load up to $\frac{1}{4} M$ new values for each sub-tile. Let the width of the tiles in $\matU_1$ be $t_1, t_2, \ldots$; note that their sum is $r$. Then the number of times we need to start a new sub-tile within a tile is at most $\frac{t_1}{\frac{M}{4g}} + \frac{t_1}{\frac{M}{4g}} + \ldots = \frac{r}{\frac{M}{4g}} = \frac{4 \cdot r \cdot g}{M}$. With ${d \choose g}$ tiles, this means that the number of sub-tiles is at most ${d \choose g} + \frac{4 \cdot r \cdot g}{M}$ altogether. This results in a total I/O cost of
\[ \frac{n}{\frac{M}{4g}} \cdot \frac{d}{g} \cdot \left( {d \choose g} + \frac{4 \cdot r \cdot g}{M} \right) \cdot M \, , \]
which is equivalent to Equation~\eqref{eq:maxbound}.

The strategy extends to the multiplication $\matU_2^{\top}\matV$ exactly as in the proof of Key Lemma~\ref{th:with_w}: we split $\matU_2^{\top}$ to column strips of width $\frac{M}{4g}$, split each of these to ${d \choose g}$ parts vertically, and form ${d \choose g} \cdot \frac{d}{g}$ tiles in the output $\matH$, each of which is aggregated over $\frac{n}{\frac{M}{4g}}$ steps. Similarly to the second multiplication, we might have tiles in $\matU_2^{\top}$ that are taller than $\frac{M}{4g}$; these are then split vertically into sub-tiles of height $\frac{M}{4g}$ at most. For all of the sub-tiles, the tile of size $g \times \frac{M}{4g}$ in $K^{\top}$ and the tile of size $\frac{M}{4g} \times g$ in $\matV$ can remain the same, whereas in $\matH$, we save different sub-tiles each time.

\subsection{Lower bound proof}

The lower bound proof of Theorem~\ref{th:case3} is similar to that in Key Lemma~\ref{th:with_w}. However, here we only focus on the columns of $\matU_1$ that have at least $\frac{g}{2}$ distinct parents in $\matQ$. One can show that if $d \geq 5g$, this indeed amounts to a constant fraction of the columns of $\matU_1$. This follows from the lemma below, which can be understood as the reverse direction: the number of terms generated from at most $(\frac{g}{2}-1)$ variables is at most $\delta \cdot r$.

\begin{lemma} \label{lem:halfcover}
For $d \geq 5g$, there exists a constant $0<\delta<1$
\[ {d \choose \frac{g}{2}-1} \cdot \tau\left(\frac{g}{2}-1\right) \leq \delta \cdot r \, . \]
\end{lemma}

Since this is only a technical lemma that offers no further insight into the proof, we defer its proof to Appendix~\ref{app:tech_lemma}.

In our proof, we only consider a subset of the matrix multiplication: the at least $\delta \cdot r$ columns in $\matU_1$ that satisfy this condition, and the corresponding at least $\delta \cdot r$ rows in $\matH$. We assume that all other I/O operations are free; the lower bound derived this way clearly carries over to the original matrix multiplication.

The proof begins similarly to Key Lemma~\ref{th:with_w}. Let $D_i$ be a partition in a $2M$-partition of this computation. Consider its minimum set and a dominator set of size at most $2M$. Let $\Gamma$ be the set of all the uncovered internal nodes in $D_i$ that are not among these $4M$ nodes. All uncovered nodes $v \in \Gamma$ must have their parent in $\matH$ in the dominator set, and also either their parent in $\matU_1$ or all the parents of this parent in the dominator set.

For any $i \in [n]$, let $w_i\,\!^{(\matQ)}$ and $w_i\,\!^{(\matU_1)}$ again be, respectively, the number of entries in the $i^{\text{th}}$ row of $\matQ$ and the $i^{\text{th}}$ row of $\matU_1$ that are in the dominator set. Let $w_i = w_i\,\!^{(\matQ)} + w_i\,\!^{(\matU_1)}$. Consider two sub-matrices of $\matU_1$: let $\matU_1^{(a)}$ contain the rows with $w_i \geq \frac{g}{2}$, and $\matU_1^{(b)}$ contain the rows with $w_i < \frac{g}{2}$. This splits the matrix multiplication into two sub-multiplications $\matU_1^{(a)} \matH$ and $\matU_1^{(b)} \matH$. In $\matU_1^{(a)}$, each row has at least $\frac{g}{2}$ nodes in the dominator set, so the number of rows is at most $\frac{2M}{\frac{g}{2}}=\frac{4M}{g}$. The number of nodes of $\matH$ in the dominator set is at most $2M$. Each row in $\matU_1^{(a)}$ has only one value multiplied with any entry in $\matH$, so the number of uncovered internal nodes in $\matU_1^{(a)} \matH$ is at most $\frac{4M}{g} \cdot 2M = \frac{8M^2}{g}$.

In $\matU_1^{(b)}$, none of the nodes in (our restricted) $\matU_1$ can have all their parents in the dominator set; hence if a node in $\Gamma$ has its parent in $\matU_1^{(b)}$, then this parent has to be in the dominator set. Thus for simplicity, we can assume that $w_i\,\!^{(\matQ)} = 0$, i.e.\ none of the source nodes are in the dominator set, since these nodes could be removed from the set anyway. With $w_i\,\!^{(\matU_1)} < \frac{g}{2}$, each of the rows in $\matU_1^{(b)}$ can have at most $\frac{g}{2}$ nodes which have a child in $\Gamma$. Similarly to before, each node in $\Gamma$ must also have a successor in the minimum set, so there are at most $2M$ sink nodes in the output matrix with a predecessor in $\Gamma$. Each sink node has at most $\frac{g}{2}$ predecessors in $\Gamma$, since number of nodes in $\matU_1^{(b)}$ with a child in $\Gamma$ is at most $\frac{g}{2}$ in the given row. This means that the number of uncovered internal nodes in $\matU_1^{(b)} \matH$ is at most $\frac{g}{2} \cdot 2M = g \cdot M$.

In Case III, we have $M = \omega(g^2)$. This means that $g \cdot M$ is also upper bounded by $O(\frac{M^2}{g})$, so the total number of nodes in $\Gamma$ is at most $4M+\frac{8M^2}{g}+g \cdot M = O(\frac{M^2}{g})$. Having $n \cdot r \cdot d$ internal nodes, the number of partitions is at least $\Omega(\frac{n \cdot r \cdot d \cdot g}{M^2})$. Using Lemma~\ref{lem_spart}, the optimal I/O cost in this case is at least $\Omega(\frac{n \cdot r \cdot d \cdot g}{M})$.

The proof can be extended to the entire computation $\matU_1 (\matU_2^\top \matV)$ exactly as in Key Lemma~\ref{th:with_w}.

\section{Proof of Theorem~\ref{th:case4}} \label{app:case4}

The upper bound in Theorem~\ref{th:case4} follows easily from the generic upper bound in Lemma~\ref{lem:static}.

For the lower bound, the proof starts identically to that of Theorem~\ref{th:case3}. We only consider the columns of $\matU_1$ with at least $\frac{g}{2}$ distinct parents in $\matU_1$. This once again amounts to a constant fraction of the columns of $\matU_1$.

For any row $i \in [n]$, let $w_i$, $w_i\,\!^{(\matQ)}$ and $w_i\,\!^{(\matU_1)}$ be as before. We split the matrix again into $\matU_1^{(a)}$ for rows with $w_i \geq \frac{g}{2}$, and $\matU_1^{(b)}$ for rows with $w_i < \frac{g}{2}$. In the sub-multiplication $\matU_1^{(a)} \matH$ the number of uncovered internal nodes in $\matU_1^{(a)} \matH$ is at most $\frac{8M^2}{g}$ as before.

Otherwise, the proof is analogously to the lower bound for standard matrix multiplication~\cite{hong1981complexity}, essentially repeating the same arguments with a split at $\sqrt{M}$ instead of $\frac{g}{2}$. Recall that for each row in $\matU_1^{(b)}$, we can assume $w_i\,\!^{(\matQ)} = 0$. Let us further divide the matrix $\matU_1^{(b)}$, with $\matU_1^{(b_1)}$ containing the rows of $\matU_1^{(b)}$ where $w_i\,\!^{(\matU_1)} \geq \sqrt{M}$, and $\matU_1^{(b_2)}$ containing the rows of $\matU_1^{(b)}$ where $w_i\,\!^{(\matU_1)} < \sqrt{M}$. In $\matU_1^{(b_1)}$, the number of rows is at most $\frac{2M}{\sqrt{M}}=2\sqrt{M}$. Each of these rows has at most one value multiplied with any of the (at most $2M$) nodes of $\matH$ in the dominator set, so the number of nodes of $\Gamma$ in $\matU_1^{(b_1)} \matH$ is at most $2M \cdot 2\sqrt{M} = 4 M \cdot \sqrt{M}$. In $\matU_1^{(b_2)}$, we again consider the at most $2M$ sink nodes in the output matrix that have a predecessor in $\Gamma$. Each of these can have at most $\sqrt{M}$ predecessors in $\Gamma$, since the number of nodes in $\Gamma$ in any row is at most $\sqrt{M}$. This means that $\Gamma$ has at most $2 M \cdot \sqrt{M}$ nodes in $\matU_1^{(b_2)}$. 

Altogether, the number of nodes of $\Gamma$ in $\matU_1^{(b)} \matH$ in this case is at most $6 M \cdot \sqrt{M}$. Together with the nodes in $\matU_1^{(a)} \matH$ and the internal nodes directly in the dominator/minimum sets, this is at most $4M + \frac{8M^2}{g} + 6 M \cdot \sqrt{M}$ nodes. Since Case IV has $M = O(g^2)$, we also have $\frac{M^2}{g} = O (M \cdot \sqrt{M})$. As such, this expression is altogether in $O(M \cdot \sqrt{M})$. With $n \cdot r \cdot d$ internal nodes, the number of partitions is $\Omega(\frac{n \cdot r \cdot d}{M \cdot \sqrt{M}})$, and the optimal I/O cost according to Lemma~\ref{lem_spart} is $\Omega(\frac{n \cdot r \cdot d}{\sqrt{M}})$.

\section{Technical lemmas} \label{app:tech_lemma}

\subsection{Using attention to compute the exponential function}
\label{appendix:attention_implies_exponential}
Assume that attention can be computed exactly in $T$ steps. To compute the exponential $\exp(x)$ for any $x$, simply set $\matQ=1$, $\matK^{\top}=\begin{pmatrix}x & 2x\end{pmatrix}$, $\matV=\begin{pmatrix}
    1; 0
\end{pmatrix}$, then compute $\matA=\Att(\matQ,\matK,\matV)=\frac{\exp(x)}{\exp(x)+\exp(2x)}$, and finally return $z=1/\matA-1=\exp(x)$. 

\subsection{Closed form for $\tau(w)$} 

For completeness, we establish the closed-form expression on $\tau(w)$.

\begin{lemma}
For any integer $w \geq 1$, we have
\[  \sum_{l=0}^{g} \binom{w+l-1}{l} = {w+g \choose g} \, . \]
\end{lemma}

\renewcommand{\proofname}{Proof}
\begin{proof}
We can show this by an induction on $g$. The claim clearly holds for small values. For $g=0$, we have ${w-1 \choose 0} = 1$ on the left-hand side and ${w \choose 0} = 0$ on the right-hand side. For $g=1$, we have $1+{w \choose 1} = w+1$ on the left-hand side and ${w+1 \choose 1} = w+1$ on the right-hand side.

Now consider $g \geq 2$, and assume the claim holds for $(g-1)$. This means that we only need to show
\[ {w+g-1 \choose g} = {w+g \choose g} - {w+(g-1) \choose (g-1)}  \, . \]
Indeed, the right-hand side is identical to
\begin{gather*}
\left( \frac{(w+g)\cdot (w+g-1) \cdot \ldots \cdot (w+1)}{g!} -  \frac{(w+g-1)\cdot (w+g-2) \cdot \ldots \cdot (w+1)}{(g-1)!} \right) = 
\\[2ex]
(w+g) \cdot \frac{(w+g-1) \cdot \ldots \cdot (w+1)}{g!} - g \cdot \frac{(w+g-1) \cdot \ldots \cdot (w+1)}{g!} = 
\\[2ex]
\frac{(w+g-1) \cdot \ldots \cdot (w+1) \cdot w}{g!} = {w+g-1 \choose g} \, .
\end{gather*}
\end{proof}

\subsection{Proof of Lemma~\ref{lem:halfcover}}

Below we prove Lemma~\ref{lem:halfcover}, i.e.\ that for $d \geq 5g$, we have a constant $0<\delta<1$ with
\[ {d \choose \frac{g}{2}-1} \cdot \tau\left(\frac{g}{2}-1\right) \leq \delta \cdot r \, . \]

\begin{proof}
For simplicity, we instead prove the stronger statement
\[ {d \choose \frac{g}{2}} \cdot \tau\left(\frac{g}{2} - 1 \right) \leq \delta \cdot r \, . \]
Using the definition of $r$ and $\tau$, this is equivalent to
\[ {d \choose \frac{g}{2}} \cdot {\frac{3}{2}g-1 \choose \frac{g}{2}-1} \leq \delta \cdot {d+g \choose g} \, . \]
This can be further expanded to
\[ \frac{d!}{(\frac{g}{2})! \cdot (d-\frac{g}{2})!} \cdot \frac{(\frac{3}{2}g-1)!}{(\frac{g}{2}-1)! \cdot g!} \leq \delta \cdot \frac{(d+g)!}{d! \cdot g!} \, \]
and then
\[ d \cdot \ldots \cdot (d-\frac{g}{2}+1) \cdot \frac{(\frac{3}{2} g -1 ) \cdot \ldots \frac{g}{2}}{(\frac{g}{2})!} \leq \delta \cdot (d+g) \cdot \ldots \cdot (d+1) \, .\]
We can upper bound $(\frac{3}{2} g - 1 )\cdot \ldots \frac{g}{2}$ by $\frac{1}{2} \cdot g^g$. We can lower bound $(\frac{g}{2})!$ by $(\frac{g}{5})^{g/2}$. We also have
\[  \frac{(d+g) \cdot \ldots \cdot (d+1)}{d \cdot \ldots \cdot (d-\frac{g}{2}+1)} \geq d^{\frac{g}{2}} \, . \]
Hence it is enough to prove that
\[ \frac{1}{2} \cdot \frac{g^g}{(\frac{g}{5})^{g/2}} = \frac{1}{2} \cdot 5^{\frac{g}{2}} \cdot g^{\frac{g}{2}} \leq \delta \cdot d^{\frac{g}{2}} \, .\]
This is equivalent to
\[ \frac{d}{5g} \geq \left(\frac{1}{4 \cdot \delta^2}\right)^{\frac{1}{g}} \, .\]
By selecting a constant $\delta$ close to $1$, we have $\left(\frac{1}{4 \cdot \delta^2}\right) \leq \frac{1}{3}$. The left side is then upper bounded by $1$ for any $g \geq 1$, and hence it is sufficient to have $d \geq 5 \cdot g$, which holds due to our assumption.

We note that if we modify the restriction condition from $\frac{g}{2}$ to a smaller linear factor of $g$, the above proof can be applied similarly; this would allow us to loosen the condition $d \geq 5 \cdot g$ to obtain a slightly smaller constant factor between $d$ and $g$, if desired.
\end{proof}

\end{document}